%% file: acl_latex.tex
\documentclass[11pt]{article}

\usepackage[final]{acl}

\usepackage{times}
\usepackage{latexsym}
\usepackage{enumitem}
\usepackage{comment}
\usepackage[T1]{fontenc}

\usepackage[utf8]{inputenc}

\usepackage{microtype}

\usepackage{inconsolata}

\usepackage{graphicx}
\usepackage{amsmath}
\usepackage{booktabs}
\usepackage{multirow}
\usepackage{array}
\usepackage{placeins}
\usepackage{colortbl}
\usepackage[prependcaption,textsize=scriptsize,color=pink]{todonotes}

\usepackage{amsfonts,amssymb,amsthm}
\newtheorem{theorem}{\textbf{Theorem}}

\DeclareRobustCommand{\name}{\textsc{ReContext}}

\title{ReContext: Recursive Evidence Replay as LLM Harness for \\Long-Context Reasoning}

\author{{\normalfont Yanjun Zhao\footnotemark[1],} {\normalfont Ruizhong Qiu\footnotemark[1],} {\normalfont Tianxin Wei\footnotemark[1],} {\normalfont Yuanchen Bei,} {\normalfont Zhining Liu,}\\
{Lingjie Chen},
{Ismini Lourentzou}, {Hanghang Tong}, {Jingrui He}$^{\dag}$\\
\vspace{-0.8em} \\
University of Illinois Urbana-Champaign\\
\texttt{ \{yanjunzh, jingrui\}@illinois.edu}
}

\begin{document}
\maketitle

{
\renewcommand{\thefootnote}{\fnsymbol{footnote}}
\footnotetext[1]{Equal contribution. $^\dag$Corresponding author.}
}

\begin{abstract}
Understanding and reasoning over long contexts has become a key requirement for deploying large language models (LLMs) in realistic applications. Although recent LLMs support increasingly long context windows, they often fail to use relevant evidence that is already present in the input, revealing a gap between context access and effective context utilization. In this work, we propose \textbf{Recursive Evidence Replay as LLM Harness for Long-Context Reasoning} (\name), a training-free inference method for improving long-context reasoning. \name{} uses model-internal relevance signals to construct a query-conditioned evidence pool and replays it before final generation while preserving the full original context. This recursive selection process separates evidence organization from answer generation without training, external memory, or context pruning. We also provide a theoretical analysis based on associative memory, which characterizes the context as a memory store, the question as a retrieval cue, attention as cue-trace association, and replay as trace reactivation. Experiments on eight long-context datasets with 128K context length show that \name{} consistently improves evidence utilization across Qwen3-4B, Qwen3-8B, and Llama3-8B, achieving the \textbf{best average rank} on all three backbones. Code is available at https://github.com/Yanjun-Zhao/ReContext.
\end{abstract}

\begin{figure}[t]
    \centering
    \includegraphics[width=0.9\columnwidth]{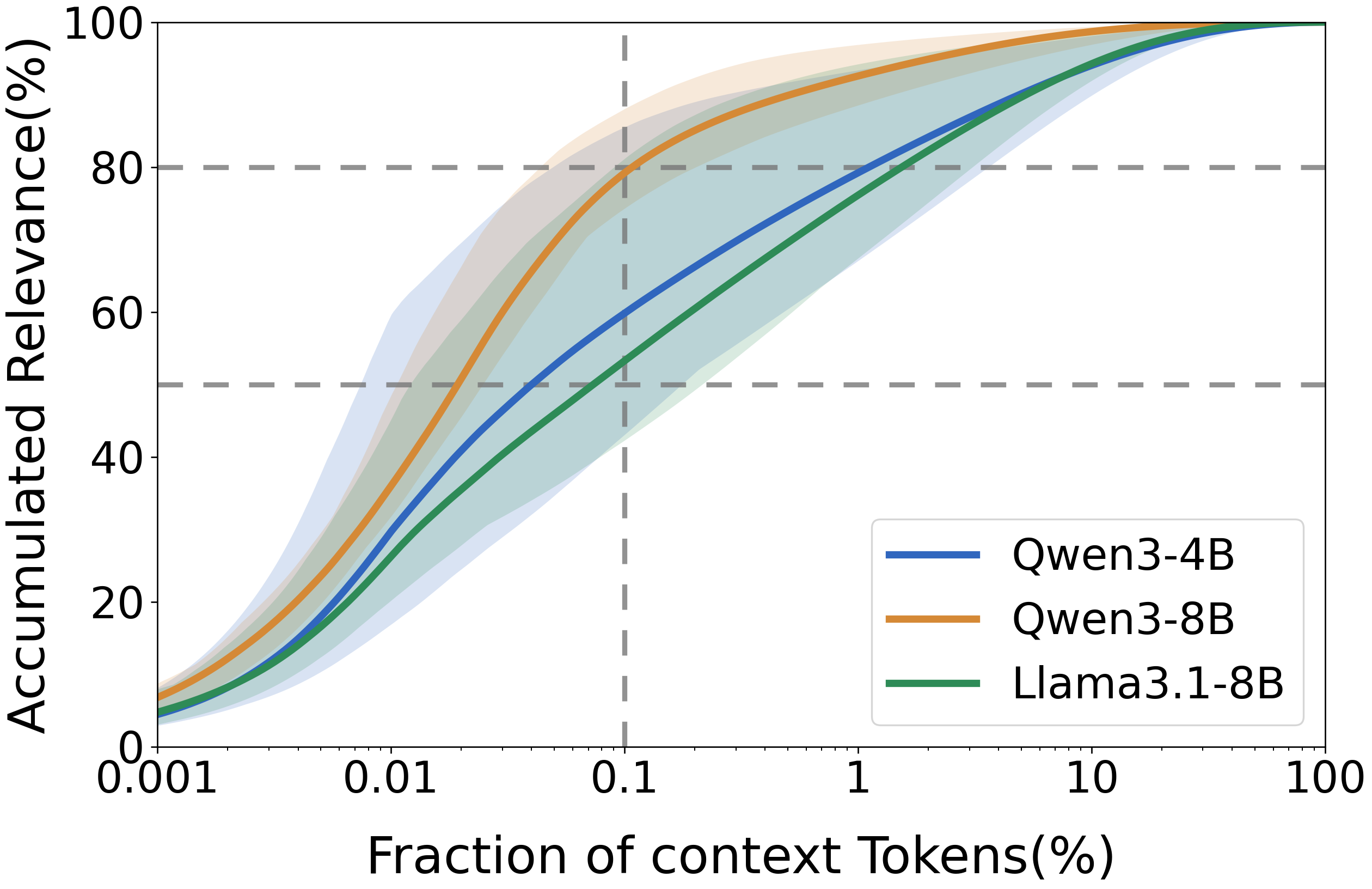}
    \caption{\textbf{Top 0.1\% of context tokens already accounts for about 50\% / 80\% accumulated relevance score across three LLMs, corresponding to only 128 tokens in a 128K-token context.} This figure ranks all context tokens by their relevance scores with respect to the question and shows how much accumulated relevance score is covered by the top-ranked tokens. Each curve represents the mean trend over eight datasets, and the shaded region shows the variance across datasets.}
    \label{fig:teaser}
\end{figure}

\input{sections/introduction}
\input{sections/related_work}

\input{sections/method}
\input{sections/experiment}
\input{sections/conclusion}

\section*{Limitations}

\name{} requires access to model-internal relevance signals, which limits its
direct use with closed-source APIs that do not expose attention or similar
scoring information. The method also adds a read-and-replay stage, so inference
latency is higher than direct full-context decoding, although it remains
training-free and does not maintain a persistent external memory. 

\bibliography{custom}

\clearpage
\appendix

\input{sections/appendix}

\end{document}

%% file: sections/introduction.tex
\section{Introduction}
\label{Introduction}

Long-context large language models (LLMs) can now place entire documents, multi-document collections, and extended dialogues inside a single prompt. However, longer context windows do not guarantee reliable long-context reasoning. A recurring failure mode is that the evidence needed to answer a question is already present in the input, but the model does not consistently use it during generation~\cite{yen2025helmet, ye2026dysco, bei2026memgallerybenchmarkingmultimodallongterm}
This suggests that the bottleneck lies not only in \emph{context access}, but also in \emph{context harnessing}: we need a mechanism that dynamically manages long contexts during reasoning by continuously identifying, organizing, and updating the information most relevant to the current stage, thereby enabling more grounded and efficient long-context reasoning\looseness-1.

\begin{figure*}[tbp]
  \centering
  \includegraphics[width=0.99\linewidth]{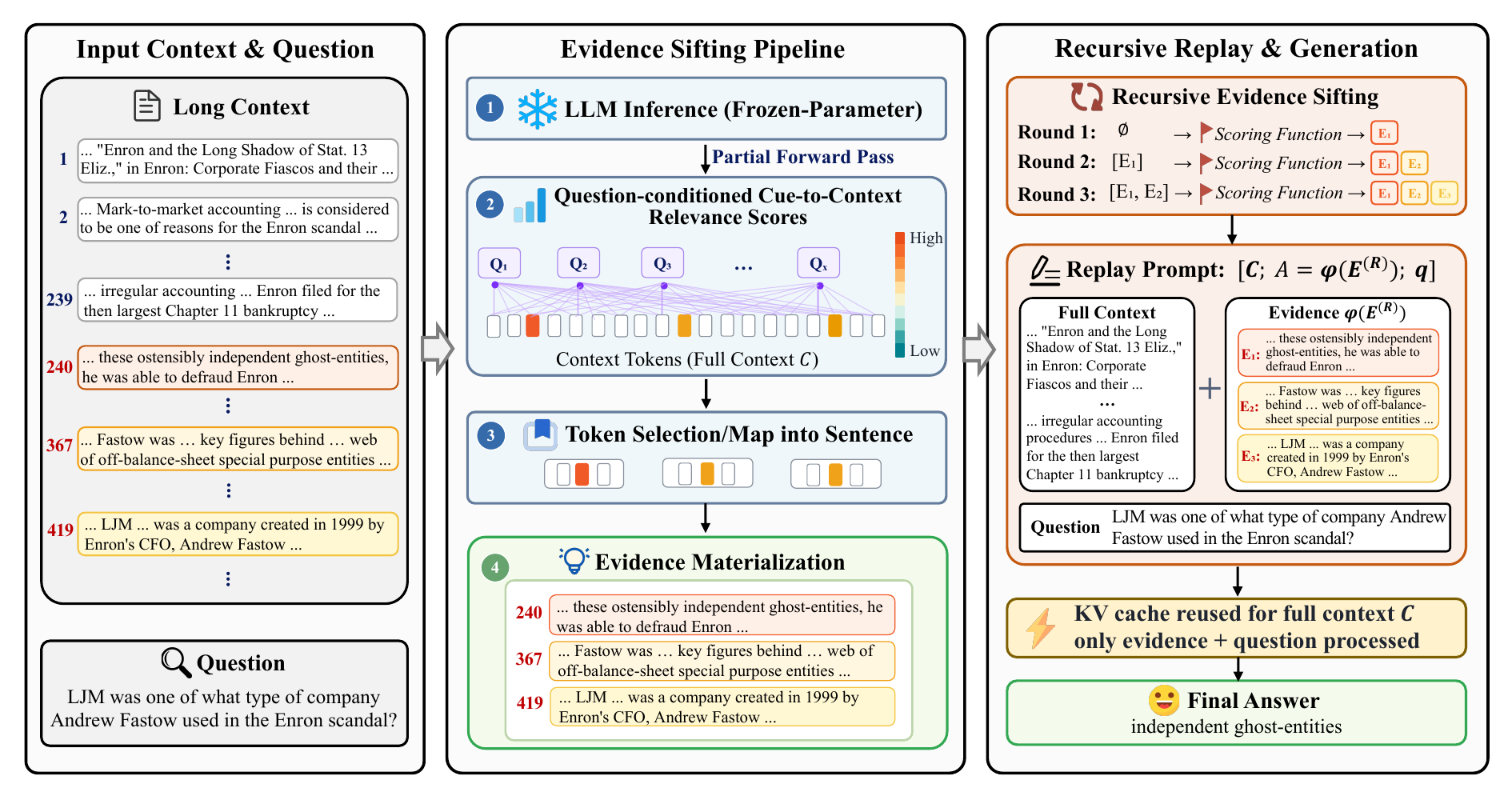}
  \vspace{-0.3cm}
  \caption{\textbf{Overview of \name{}.} \name{} identifies question-relevant evidence from a long context using internal LLM relevance signals, materializes selected tokens into grounded evidence spans, and recursively replays the resulting evidence before final generation while preserving access to the full context.}
  
  \label{fig:framework}
\end{figure*}

Standard long-context prompting requires a model to read the context and answer the query in a single inference pass. As the context grows, the evidence must compete with increasing irrelevant information, making it harder for the model to ground its answers and leading to errors or hallucinations~\citep{li2024snapkvllmknowslooking,liu2025selfelicit}. Recent efforts address this problem from different angles: attention intervention methods modify low-level model behavior~\citep{li2024snapkvllmknowslooking,tang2024questqueryawaresparsityefficient,ye2026dysco}, making them invasive because they require changing the backbone forward or decoding logic. Retrieval and external-memory methods add retrieval systems~\citep{lewis2021retrieval,xu2025amem}, while compression methods shorten the effective input~\citep{jiang2023llmlingua,jiang2024longllmlingua,zhao2025dac}, both often rely on retrieved, compressed, or LLM-summarized evidence views, which may lose fine-grained details and become unstable on complex multi-hop tasks. This suggests that a context harness should not be limited to retrieving or reducing the input. Instead, while preserving full access to the original input, it should dynamically query and maintain a repository of supporting evidence according to the current reasoning stage, helping the model perform high-quality reasoning over long contexts.\looseness-1

We propose \textbf{Recursive Evidence Replay as LLM Harness for Long-Context Reasoning} (\name), a training-free inference method for long-context reasoning. Given a long context and a question, \name{} reads the original prompt, uses question-conditioned internal attention as candidate evidence proposals, materializes selected tokens as grounded text spans, and replays these spans before final answer generation. As shown in Figure~\ref{fig:teaser}, 128 tokens in a 128K-token context, already account for roughly 50–80\% of accumulated question-conditioned relevance. \name{} turns these sparse signals into an evidence pool: across a small number of rounds, \name{} updates an ordered evidence pool by conditioning each new selection step on the original context, the question, and the evidence pool accumulated so far. This scaffold serves as a temporary external workspace in an iterative pipeline, which is why we view ReContext as LLM harness between full-context reading and final generation. The replayed scaffold changes the model state from which the next round's query-token attention scores are computed, allowing later rounds to surface evidence related to previously selected spans. The full context remains in the prompt; the evidence pool is used for emphasis, not exclusion. We use ``recursive'' in this limited inference-time sense: each evidence proposal depends on the evidence pool produced by previous rounds, rather than on an open-ended reasoning loop.\looseness-1

We further provide theoretical insights for \name{} from the perspective of associative memory. The long context can be viewed as a memory store, the question as a retrieval cue, attention as a prompt-internal proxy for cue-trace association, and replay as reactivation of selected traces near generation time.  This view yields a monotonic-improvement proof showing that recursive evidence replay can move the hidden representation toward the answer embedding, and frames \name{} as query-evidence rebinding with internal relevance signals rather than a trained retriever or context pruning.

\noindent In summary, the contributions of our work are:
\begin{itemize}[itemsep=0.5ex, parsep=0pt, topsep=-2.3pt, leftmargin=0.4cm]
    \item We introduce \textbf{Recursive Evidence Replay as LLM Harness} (\name), a training-free method that converts prompt-internal relevance signals into an explicit recursive evidence pool while preserving the full context.

    \item We provide an associative-memory explanation of \name{} and a monotonic-improvement proof, formalizing evidence selection as cue-trace association and replay as trace reactivation for query-evidence rebinding.
    
    \item We evaluate \name{} on eight 128K long-context datasets across Qwen3-4B, Qwen3-8B and Llama3-8B, where it achieves the best average rank on all three backbones and improves mean accuracy over Vanilla from 0.24 to 0.30, a \textbf{24.6\% relative gain}.
\end{itemize}

%% file: sections/related_work.tex
\section{Related Work}
\label{Related_work}
\subsection{Long-Context Utilization and Evidence-Guided Reasoning}

Recent work has significantly extended the context window of large language models (LLMs), enabling them to process long documents, multi-document inputs, code repositories, and extended dialogue histories. Existing methods improve long-context modeling through position extrapolation, efficient attention mechanisms, and long-context fine-tuning~\cite{peng2023yarn, chen2024longlora, ding2023longnet, ding2024longrope}. Along with these modeling advances, a series of benchmarks have been proposed to evaluate long-context capabilities across diverse scenarios, including document question answering, multi-document reasoning, retrieval, summarization, code completion, and synthetic stress tests~\cite{bai2024longbench, shaham2023zeroscrolls, an2023leval, zhang2024infinitebench, hsieh2024ruler, yen2025helmet}. 
These studies show that long-context evaluation should not only measure whether a model can accept long inputs, but also whether it can locate and use relevant information hidden in long contexts. 

However, increasing the context window alone does not guarantee effective context utilization. Prior analysis has shown that LLMs are sensitive to the position of relevant evidence and may fail to use information when it appears in less favorable locations within the prompt~\cite{liu2023lost}. To address this issue, retrieval-augmented generation retrieves relevant passages before generation~\cite{lewis2021retrieval}, while prompt and context compression methods reduce the input length by filtering, pruning, or compressing less informative content~\cite{li2023selective, jiang2023llmlingua, jiang2024longllmlingua}. Another line of work improves inference efficiency through KV cache compression or token eviction~\cite{liu2023scissorhand, zhang2023h2oheavyhitteroracleefficient, xiao2024efficientstreaminglanguagemodels}.\looseness-1

\subsection{Internal Attention Signals and Memory-Based Interpretation}

Another related direction studies token-level importance inside LLMs. Prior work has observed that attention patterns in long-context inference are often highly structured: a small number of tokens may act as heavy hitters, attention sinks, or salient key positions that contribute disproportionately to future attention computation~\cite{liu2023scissorhand, zhang2023h2oheavyhitteroracleefficient, xiao2024efficientstreaminglanguagemodels}. Based on this observation, recent methods select or retain important KV cache tokens to improve inference efficiency while preserving model performance~\cite{li2024snapkvllmknowslooking, cai2025pyramidkvdynamickvcache, tang2024questqueryawaresparsityefficient, feng2025adakvoptimizingkvcache}.\looseness-1 

These studies demonstrate that model-internal signals can provide useful information about which context tokens are important. Our work is related to this line of research, but differs in objective.
Rather than proposing a new token-importance estimator or KV-cache compression strategy, \name{} uses existing question-to-context relevance signals to construct grounded candidate evidence spans, and studies how replaying these spans can improve long-context answer generation. The selected evidence is question-conditioned and chosen according to its relevance to the current question, rather than its general contribution to maintaining the cache. This use of relevance signals also admits an associative-memory interpretation: selected spans are treated as context traces associated with the query and replayed before generation. Classical and modern associative memory models retrieve stored patterns using partial or noisy cues, and recent studies have connected Transformer attention with memory retrieval mechanisms~\cite{ramsauer2021hopfieldnetworksneed, krotov2025modernmethodsassociativememory}.

%% file: sections/method.tex
\section{Method}
\label{Method}

\input{tables/main_results}

\subsection{Overview}

Given a long context $C$ and a question $q$, a standard long-context LLM generates directly from $[C;q]$. \textbf{Context Harness with Recursive Evidence Selection} (\name) instead separates evidence organization from answer generation. It first reads the original prompt, extracts candidate evidence spans using question-conditioned internal relevance signals, replays these spans as an evidence pool, and then generates the final answer from the full context, the evidence pool, and the question.
The method does not prune the prompt or directly modify attention logits during final decoding. The original context remains available throughout generation. The replayed scaffold emphasizes candidate evidence and makes evidence utilization explicit while keeping unselected context accessible.

\name{} maintains an ordered evidence pool and updates it over a small fixed number of rounds. Each round computes relevance scores from a prompt that already contains the evidence pool accumulated in previous rounds.

\subsection{Evidence Selection}

Let $M$ denote the backbone LLM. For a current prompt $x$, let $\mathcal{I}_x$ denote token positions in $x$, and let $\mathcal{I}_C\subseteq\mathcal{I}_x$ denote positions belonging to the original context $C$. Let $\mathcal{Q}(x)=(t_1,\ldots,t_L)$ be the last $L\le w$ cue positions in the prompt suffix, with $w=8$ in our main experiments. These suffix cues provide a query-conditioned readout, and in later rounds are conditioned on the replayed scaffold. During a read pass over $x$, we score prompt tokens by aggregating attention from these cue tokens over selected heads. For cue position $t_u$, define
\begin{equation}
    a^{(u)}_i =
    \frac{1}{|\mathcal{H}|}
    \sum_{(l,h)\in\mathcal{H}} A^{(l,h)}_{t_u,i},
    \qquad i\in\mathcal{I}_x ,
\end{equation}
where $\mathcal{H}$ is a selected set of layer-head pairs, and $A^{(l,h)}_{t_u,i}$ is the attention weight from cue token $t_u$ to token $i$ at layer-head pair $(l,h)$. We accumulate cue-token evidence across $\mathcal{Q}(x)$ with exponential decay and normalization to obtain the final relevance score $r_i$. We then restrict candidates to the original context and select the top-$K$ positions:
\begin{equation}
    P = \operatorname{TopK}_{i\in\mathcal{I}_C}(r_i;K).
\end{equation}
Here $K$ is the evidence-token budget, and $P$ contains the selected context positions.
We treat these scores as an inexpensive prompt-internal proposal signal, which can identify candidate spans that may be useful for answer generation.

\subsection{Evidence Materialization and Replay}

Token-level proposals are often too fragmentary for answer generation. A selected token may identify an entity, date, or predicate, but the model usually needs the surrounding statement to use it reliably. Therefore, \name{} maps selected token positions back to their containing sentences or local spans. Let $\mathcal{S}_C=(s_1,\ldots,s_N)$ be the ordered span decomposition of $C$, and let $\mathrm{pos}(s_n)$ denote the token positions covered by span $s_n$. The evidence pool is the ordered subsequence of spans touched by selected tokens:
\begin{equation}
    E = (s_n : \mathrm{pos}(s_n)\cap P\neq\emptyset)_{n=1}^{N}.
\end{equation}
Each evidence unit is copied from the original prompt, so the evidence pool is grounded rather than freely generated.

For a single-pass replay, the prompt is
\begin{equation}
    x^+ = [C; \phi(E); q],
\end{equation}
where $\phi(E)$ is the textual replay format. The answer is generated as
\begin{equation}
    y \sim M(\cdot \mid x^+).
\end{equation}
Replay places selected evidence near the question while keeping the full context available. In this sense, \name{} selects for emphasis rather than exclusion: unselected context remains available during final generation.

\subsection{Recursive Evidence Selection}

For $R$ rounds, \name{} updates the evidence pool recursively. Let $E^{(0)}=\emptyset$ be the initial ordered evidence pool. At round $j \in \{1,\ldots,R\}$, the current prompt is
\begin{equation}
    x^{(j-1)} = [C; \phi(E^{(j-1)}); q].
\end{equation}
The model reads $x^{(j-1)}$, so the replayed scaffold conditions the hidden states from which query-side prompt-suffix attention scores are computed. In the main setting, candidate positions are selected only from the original context token positions $\mathcal{I}_C$; the replayed scaffold conditions scoring but is not treated as a source of new copied spans. \name{} then obtains relevance scores $r^{(j)}$, selects positions $P^{(j)}$, materializes spans from the original context, and appends only spans that are not already in the evidence pool. The span list proposed at round $j$ is\looseness-1
\begin{equation}
    \widehat{S}^{(j)}
    =
    (s_n : \mathrm{pos}(s_n)\cap P^{(j)}\neq\emptyset)_{n=1}^{N}.
\end{equation}
The newly added spans are the ordered subsequence
\begin{equation}
    \Delta E^{(j)}=\widehat{S}^{(j)}\setminus E^{(j-1)},
\end{equation}
and the evidence pool is updated by ordered concatenation:
\begin{equation}
    E^{(j)}=E^{(j-1)}\cup\Delta E^{(j)} .
\end{equation}
Here $\oplus$ denotes ordered concatenation; we avoid set union because the evidence pool is an ordered list. The final answer is generated from
\begin{equation}
    x^{(R)} = [C; \phi(E^{(R)}); q].
\end{equation}
\begin{equation}
    y \sim M(\cdot \mid x^{(R)}).
\end{equation}

We call this process recursive because each round conditions on the evidence pool produced by previous rounds. In practice, $R$ is small and fixed, so \name{} remains a lightweight inference-time wrapper rather than an open-ended reasoning procedure.

\subsection{Theoretical Analysis}

In this subsection, we provide theoretical underpinning that our recursive evidence replay process in \name{} can  push the hidden embedding toward the answer. Our result is formally stated in Theorem~\ref{thm:similarity} below.

\begin{theorem}[monotonic improvement]
\label{thm:similarity}
Our theoretical setup is similar to prior works \cite{nichani2025understanding,olsson2022context} and stated in Appendix~\ref{app:thm:setup}. Let $h^{(j)}$ denote the hidden embedding after the $j$-th evidence replay step, and let $y$ denote the embedding of the answer. Then, for every step $j\ge1$,
\begin{align}
\cos(h^{(j)},y)>\cos(h^{(j-1)},y),
\end{align}
where $\cos(\cdot,\cdot)$ denotes cosine similarity.
\end{theorem}

The proof is involved and thus deferred to Appendix~\ref{app:thm:proof}. Intuitively, our Theorem~\ref{thm:similarity} shows that explicitly adding the evidence to the prompt can monotonically increase the similarity of the hidden embedding between the prediction and the answer.\looseness-1 

%% file: tables/main_results.tex
\begin{table*}[t]
\centering
\small
\definecolor{bestscorebg}{HTML}{F9CB9C}
\definecolor{secondscorebg}{HTML}{FCE5CD}
\newcommand{\bestscore}[1]{\cellcolor{bestscorebg}\textbf{#1}}
\newcommand{\secondscore}[1]{\cellcolor{secondscorebg}\underline{#1}}
\setlength{\tabcolsep}{1.8pt}
\caption{Main benchmark comparison across long-context tasks. \name{} achieves the best average rank for all three backbones, showing consistent gains across datasets and model scales. Darker and lighter backgrounds indicate the best and second-best results within each backbone, respectively.}
\label{tab:main_results}
\vspace{-0.2cm}
\resizebox{1.0\textwidth}{!}{
\begin{tabular}{@{}ll@{\hspace{4pt}}*{12}{>{\centering\arraybackslash}p{0.68cm}}>{\centering\arraybackslash}p{0.90cm}>{\centering\arraybackslash}p{0.76cm}>{\centering\arraybackslash}p{0.95cm}@{}}
\toprule
\multirow{2}{*}{\textbf{Model}} & \multirow{2}{*}{\textbf{Method}} &
\multicolumn{2}{c}{\shortstack{\textbf{NQ}\\\textbf{128K}}} &
\multicolumn{2}{c}{\shortstack{\textbf{TriviaQA}\\\textbf{128K}}} &
\multicolumn{2}{c}{\shortstack{\textbf{HotpotQA}\\\textbf{128K}}} &
\multicolumn{2}{c}{\shortstack{\textbf{PopQA}\\\textbf{128K}}} &
\multicolumn{2}{c}{\shortstack{\textbf{NarrQA}\\\textbf{128K}}} &
\multicolumn{2}{c}{\shortstack{\textbf{InfQA}\\\textbf{128K}}} &
\shortstack{\textbf{InfMC}\\\textbf{128K}} &
\shortstack{\textbf{Clip.}\\\textbf{128K}} &
\multirow{2}{*}{\shortstack{\textbf{Avg}\\\textbf{Rank}}} \\
\cmidrule(lr){3-4}
\cmidrule(lr){5-6}
\cmidrule(lr){7-8}
\cmidrule(lr){9-10}
\cmidrule(lr){11-12}
\cmidrule(lr){13-14}
\cmidrule(lr){15-15}
\cmidrule(lr){16-16}
& & Acc & F1 & Acc & F1 & Acc & F1 & Acc & F1 & Acc & F1 & Acc & F1 & Acc & Acc \\
\midrule
\multirow{6}{*}{Qwen3-4B}
& Vanilla & \underline{0.02} & \underline{0.21} & 0.04 & 0.24 & 0.00 & 0.10 & 0.00 & 0.11 & 0.02 & 0.17 & 0.09 & 0.21 & \secondscore{0.51} & 0.38 & 4.39 \\
& AttnSharp & \underline{0.02} & \underline{0.21} & 0.02 & 0.23 & 0.00 & 0.09 & 0.00 & 0.10 & 0.03 & 0.17 & 0.10 & \secondscore{0.23} & 0.47 & 0.42 & 4.25 \\
& DySCO & \underline{0.02} & \underline{0.21} & 0.10 & 0.30 & 0.03 & 0.13 & 0.00 & 0.10 & 0.01 & 0.17 & \secondscore{0.11} & \secondscore{0.23} & 0.50 & 0.44 & 4.00 \\
& A-MEM & \underline{0.02} & 0.20 & 0.19 & 0.37 & 0.06 & 0.15 & \secondscore{0.06} & \secondscore{0.16} & 0.04 & 0.19 & 0.07 & 0.18 & 0.43 & \secondscore{0.48} & \secondscore{3.57} \\
& DAC & \underline{0.02} & 0.18 & \secondscore{0.21} & \secondscore{0.38} & \secondscore{0.07} & \secondscore{0.17} & 0.01 & 0.09 & \secondscore{0.05} & \secondscore{0.20} & 0.07 & 0.19 & 0.43 & 0.24 & 3.79 \\
& \textbf{\name} & \bestscore{0.08} & \bestscore{0.25} & \bestscore{0.30} & \bestscore{0.45} & \bestscore{0.08} & \bestscore{0.19} & \bestscore{0.07} & \bestscore{0.19} & \bestscore{0.07} & \bestscore{0.21} & \bestscore{0.12} & \bestscore{0.24} & \bestscore{0.55} & \bestscore{0.52} & \bestscore{1.00} \\
\midrule
\multirow{6}{*}{Qwen3-8B}
& Vanilla & 0.06 & 0.26 & 0.53 & 0.66 & 0.18 & 0.31 & 0.18 & 0.32 & \secondscore{0.19} & 0.33 & 0.22 & 0.36 & \bestscore{0.64} & 0.30 & 3.96 \\
& AttnSharp & 0.06 & 0.26 & 0.43 & 0.59 & 0.18 & 0.31 & 0.13 & 0.29 & 0.18 & \secondscore{0.34} & \secondscore{0.23} & \secondscore{0.37} & 0.60 & 0.32 & 4.50 \\
& DySCO & 0.09 & \secondscore{0.30} & 0.51 & 0.64 & 0.19 & 0.31 & \secondscore{0.22} & \secondscore{0.35} & 0.18 & 0.33 & \secondscore{0.23} & 0.36 & \secondscore{0.63} & \bestscore{0.34} & \secondscore{3.25} \\
& A-MEM & 0.08 & 0.28 & 0.58 & 0.67 & \secondscore{0.21} & 0.32 & 0.19 & 0.30 & 0.17 & 0.31 & 0.18 & 0.28 & 0.58 & 0.18 & 4.21 \\
& DAC & \secondscore{0.12} & 0.29 & \secondscore{0.65} & \secondscore{0.72} & \bestscore{0.22} & \bestscore{0.35} & 0.15 & 0.30 & 0.14 & 0.28 & 0.15 & 0.26 & \bestscore{0.64} & 0.20 & 3.61 \\
& \textbf{\name} & \bestscore{0.13} & \bestscore{0.33} & \bestscore{0.68} & \bestscore{0.75} & 0.20 & \secondscore{0.34} & \bestscore{0.23} & \bestscore{0.36} & \bestscore{0.21} & \bestscore{0.35} & \bestscore{0.25} & \bestscore{0.39} & \secondscore{0.63} & \secondscore{0.33} & \bestscore{1.46} \\
\midrule
\multirow{6}{*}{Llama3-8B}
& Vanilla & 0.15 & 0.29 & \secondscore{0.69} & \secondscore{0.76} & \secondscore{0.24} & \bestscore{0.39} & \secondscore{0.21} & \secondscore{0.28} & 0.13 & 0.27 & 0.15 & \secondscore{0.34} & 0.56 & 0.32 & \secondscore{3.25} \\
& AttnSharp & 0.15 & 0.28 & \secondscore{0.69} & \secondscore{0.76} & 0.23 & \bestscore{0.39} & 0.20 & \secondscore{0.28} & 0.13 & 0.27 & 0.17 & \secondscore{0.34} & \secondscore{0.57} & 0.22 & 3.29 \\
& DySCO & 0.10 & 0.26 & 0.63 & 0.72 & 0.23 & 0.37 & 0.18 & 0.27 & 0.13 & 0.26 & 0.14 & 0.33 & 0.56 & \secondscore{0.38} & 4.57 \\
& A-MEM & \secondscore{0.16} & \bestscore{0.34} & 0.67 & 0.75 & \secondscore{0.24} & 0.34 & 0.17 & 0.23 & 0.15 & \secondscore{0.29} & \secondscore{0.19} & 0.33 & 0.54 & 0.32 & 3.43 \\
& DAC & 0.06 & 0.23 & 0.56 & 0.68 & 0.16 & 0.30 & 0.15 & 0.22 & \secondscore{0.16} & \bestscore{0.30} & 0.15 & 0.28 & 0.49 & 0.28 & 5.18 \\
& \textbf{\name} & \bestscore{0.19} & \secondscore{0.31} & \bestscore{0.70} & \bestscore{0.77} & \bestscore{0.25} & \bestscore{0.39} & \bestscore{0.22} & \bestscore{0.29} & \bestscore{0.17} & \secondscore{0.29} & \bestscore{0.22} & \bestscore{0.40} & \bestscore{0.64} & \bestscore{0.40} & \bestscore{1.29} \\
\bottomrule
\end{tabular}
}
\end{table*}

%% file: sections/experiment.tex
\section{Experiments}
\label{Experiment}

\subsection{Datasets}

\begin{figure*}[t]
    \centering
    \begin{minipage}[t]{0.45\textwidth}
        \centering
        \includegraphics[width=\linewidth]{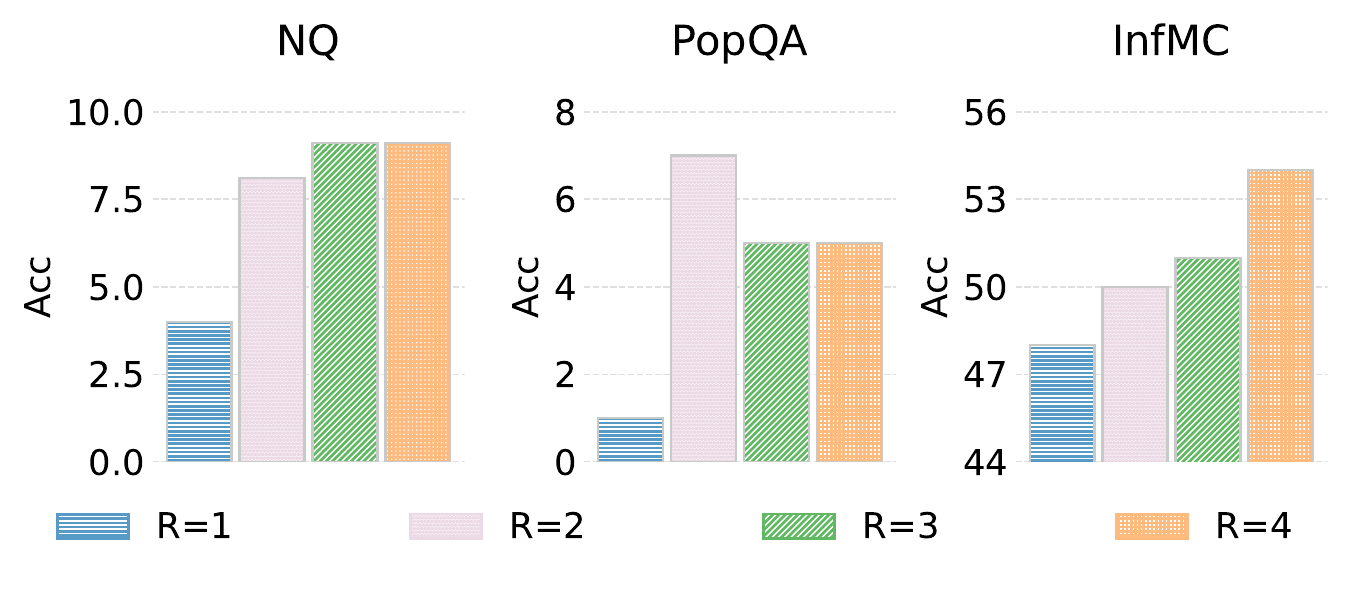}
    \end{minipage}
    \hfill
    \begin{minipage}[t]{0.45\textwidth}
        \centering
        \includegraphics[width=\linewidth]{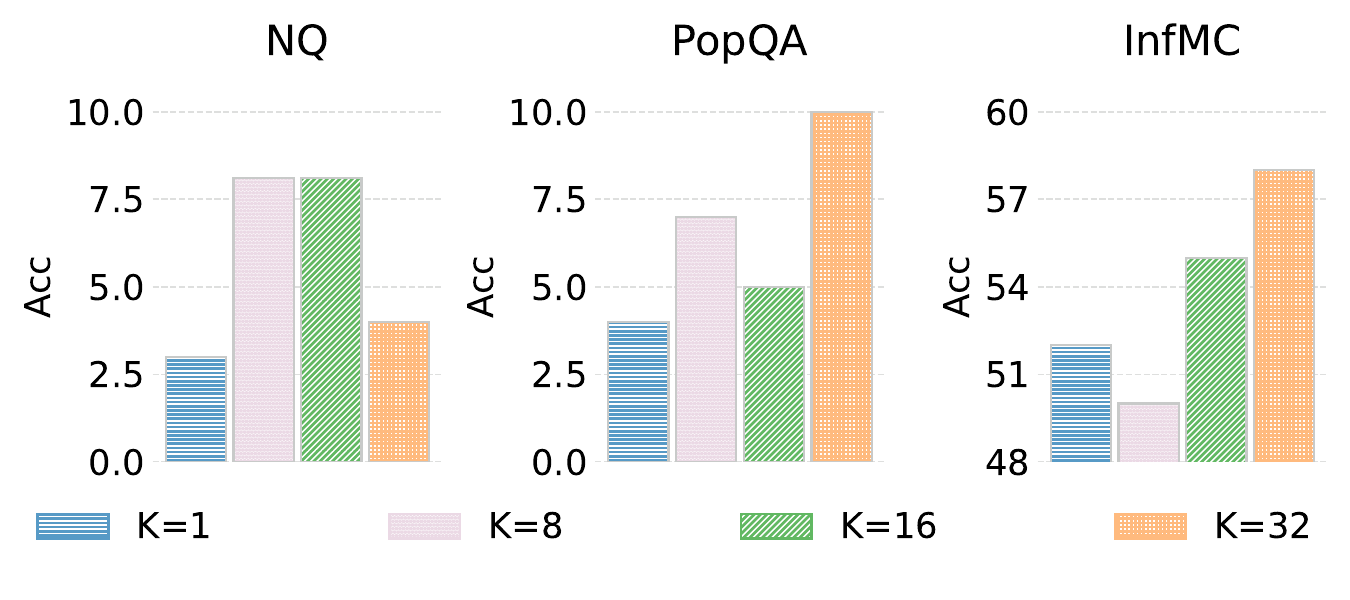}
    \end{minipage}
    \caption{Visualization of the main ablation studies. Left: the effect of recursive evidence-selection rounds $R$. Right: the effect of the top-$K$ evidence-token candidate budget.}
    \label{fig:ablation_hyperparameters}
\end{figure*}

We evaluate \name{} on eight long-context benchmarks: Natural Questions (NQ), TriviaQA, HotpotQA, PopQA, NarrativeQA, InfBench QA, InfBench MC, and CLIPPER. For the first seven datasets, we adopt the 128K-context versions constructed by HELMET~\cite{yen2025helmet}, while CLIPPER~\cite{pham2025clipper} evaluates evidence-grounded claim verification over long book contexts. These tasks cover factual question answering, multi-hop reasoning, narrative understanding, multiple-choice reasoning, and long-context claim verification. Detailed dataset descriptions are provided in Appendix~\ref{sec:appendix_datasets}.

We use the official metric associated with each task. For NQ, TriviaQA, HotpotQA, PopQA, NarrativeQA, and InfBench QA, we report answer accuracy (Acc) and token-level F1. For InfBench MC and CLIPPER, we report accuracy. Together, these tasks provide a broad testbed for long-context utilization.

\subsection{Baselines}

We compare \name{} with the following:
\begin{itemize}[itemsep=0.5ex, parsep=0pt, topsep=-2.3pt, leftmargin=0.2cm]
    \item \textbf{Vanilla} directly generates from the full context and question.
    \item \textbf{AttnSharp} sharpens attention toward question-relevant context tokens, following attention-based long-context utilization methods~\citep{tang2024questqueryawaresparsityefficient,ye2026dysco}.
    \item \textbf{DySCO} dynamically rescales decoding attention using retrieval-head signals~\citep{ye2026dysco}.
    \item \textbf{A-MEM} stores and retrieves task-relevant context evidence with an external agentic memory module~\citep{xu2025amem}.
    \item \textbf{DAC} applies dynamic attention-aware prompt compression before generation~\citep{zhao2025dac}.
\end{itemize}
In contrast, \textbf{\name{}} preserves the full original context and replays a query-conditioned evidence pool before final generation.

\subsection{Experimental Settings}

We evaluate three backbone LLMs: Qwen3-4B, Qwen3-8B, and Llama3.1-8B. For each backbone, all methods use the same prompting format, decoding configuration, and context budget. Unless otherwise specified, all methods receive the same original long-context input and question, and the main results are obtained with thinking disabled. 
Detailed experimental settings are provided in Appendix~\ref{sec:appendix_settings}.

\begin{figure*}[t]
    \centering
    \includegraphics[width=\textwidth]{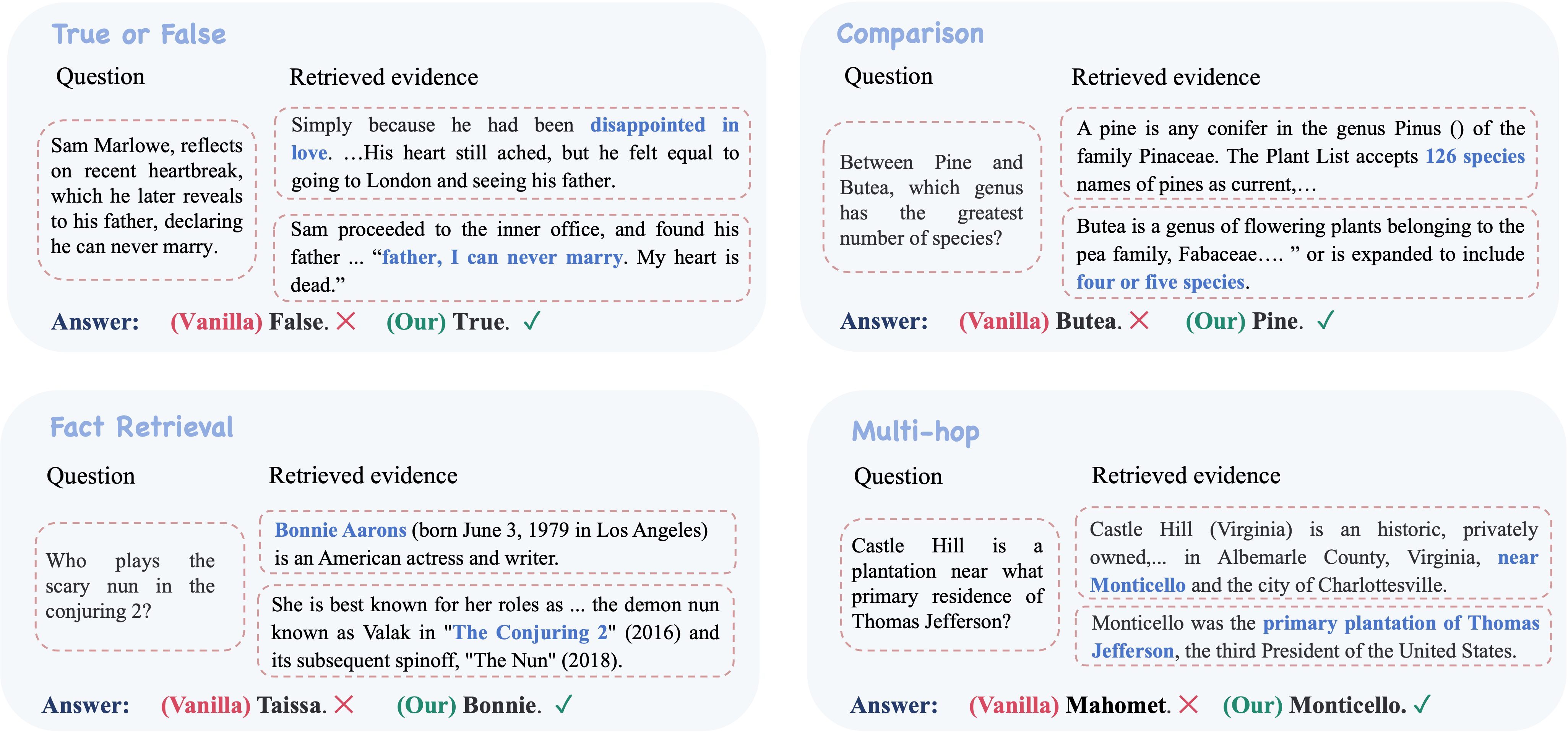}

    \caption{\textbf{Qualitative examples of \name{} evidence replay.} \name{} selects and replays query-relevant evidence spans (blue text) across diverse long-context reasoning tasks, enabling the model to ground its answer in the highlighted support and correct errors made by Vanilla generation.}
    \label{fig:case_study}
\end{figure*}

\subsection{Main Results}
The main comparison across eight datasets and three backbone models is summarized in Table~\ref{tab:main_results}. Task scores are reported as fractions in $[0,1]$, while average rank is computed by ranking methods within each backbone on each reported metric column and averaging the resulting ranks; lower is better. \name{} obtains the best average rank for all three backbones, with average ranks of 1.00 on Qwen3-4B, 1.46 on Qwen3-8B, and 1.29 on Llama3-8B. Averaging the eight Acc columns across all three backbones, \name{} improves over Vanilla from 0.24 to 0.30, a \textbf{relative gain of 24.6\%}.
On Qwen3-4B, \name{} achieves the best score on every reported metric, including improving NQ Acc from the strongest baseline score of 0.02 to 0.08. 

On Qwen3-8B, \name{} leads on most QA metrics and has the best average rank, with exceptions on HotpotQA, InfBench MC, and CLIPPER. On Llama3-8B, \name{} obtains the best average rank and the highest accuracy score for every task, while NQ, HotpotQA, and NarrativeQA F1 are led by other baselines. Overall, these results indicate that explicit evidence replay improves aggregate long-context performance across model families, without implying dominance on every individual metric.

Under a shorter 64K context budget with thinking disabled, Table~\ref{tab:robustness_64k_ablation} shows that \name{} stays within the top two on every reported metric. It ties for the best NQ Acc, obtains the best NQ F1, PopQA Acc, and InfBench MC accuracy, and ranks second on PopQA F1. Its macro-average over the five reported scores improves over Vanilla from 0.21 to 0.28, corresponding to a \textbf{35.0\% relative gain} based on the reported scores. These results suggest that the observed improvements are not tied to a single context budget.

\input{tables/thinking_ablation}

\input{tables/robustness_64k_ablation}

When thinking is enabled on Qwen3-4B, the robustness results in Table~\ref{tab:thinking_ablation} show that \name{} achieves the best NQ Acc and F1, the best PopQA Acc, and the best InfBench MC accuracy. Its macro-average over the five reported scores improves over Vanilla from 28.0 to 32.6, a \textbf{relative gain of 16.7\%}. PopQA F1 is an exception, where DySCO obtains the highest score. These results suggest that Recursive Evidence Selection remains useful even when the backbone model is allowed to produce intermediate reasoning.

\input{tables/token_source_ablation}

\begin{figure}[t]
    \centering
    \includegraphics[width=\columnwidth]{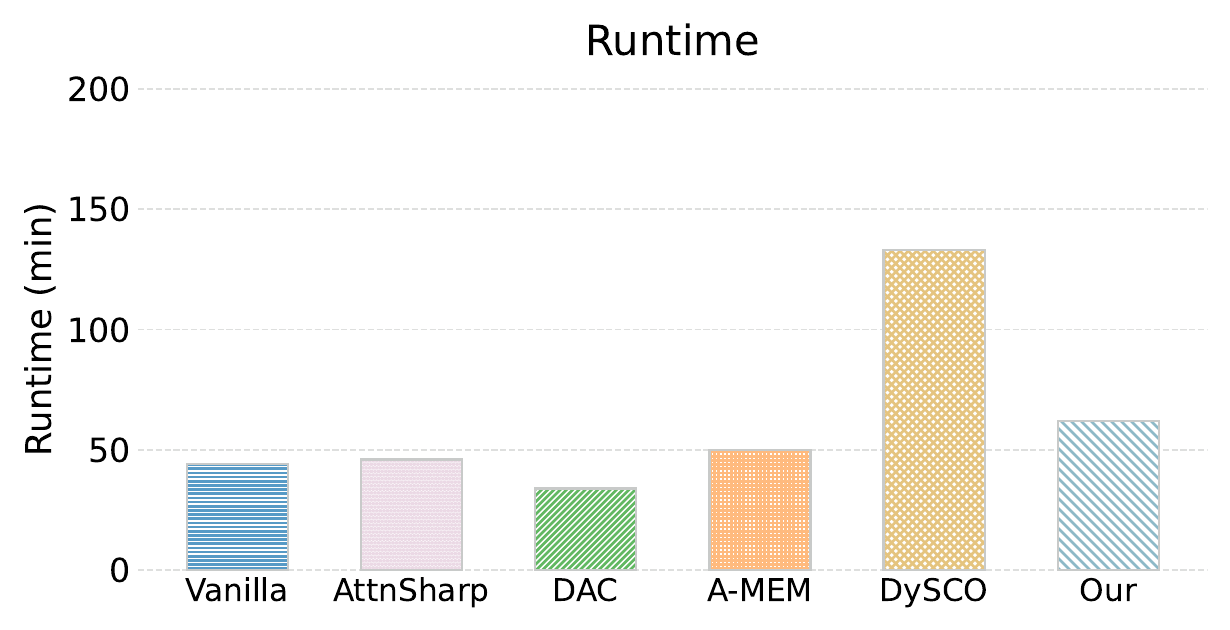}
    \caption{Runtime comparison on CLIPPER using Llama3-8B at 128K context length.}
    \label{fig:runtime_comparison}
\end{figure}

\subsection{Ablation Studies}

We analyze the effects of recursive selection rounds and candidate-token budget in Tables~\ref{tab:evidence_turn_ablation} and~\ref{tab:topk_ablation}, with selected accuracy trends visualized in Figure~\ref{fig:ablation_hyperparameters}.

The number of recursive evidence-selection rounds $R$ controls how many times the evidence pool is expanded before final replay. Moving from one round to two rounds improves all reported metrics, raising the macro-average from 0.17 to 0.22. Larger values provide further gains for NQ and InfBench MC, with the best NQ Acc at $R=3$ and $R=4$, the best NQ F1 and InfBench MC accuracy at $R=4$, and the best PopQA scores at $R=2$. This suggests that additional selection can improve performance on some tasks, but the best recursion depth is task-dependent rather than uniformly larger.

The top-$K$ token candidate budget determines how many high-scoring token positions can seed evidence span materialization in each selection round. With $R=2$ fixed, Table~\ref{tab:topk_ablation} shows that NQ Acc ties at $K=8$ and $K=16$, while NQ F1 peaks at $K=8$; by contrast, PopQA and InfBench MC achieve their strongest scores at $K=32$. The macro-average rises from 0.19 at $K=1$ to 0.23 at $K=32$, but the task-level pattern is not monotonic: larger candidate sets can expose more candidate spans, while smaller budgets can be cleaner for NQ.\looseness-1

Finally, the token-source ablation in Table~\ref{tab:token_source_ablation} compares selecting evidence tokens from the original context only versus from the full replay prompt. Context-only selection consistently outperforms full-prompt selection on NQ, PopQA, and InfBench MC, improving the macro-average from 0.19 to 0.23. The largest gains appear on PopQA Acc, which increases from 0.02 to 0.07, and NQ Acc, which increases from 0.04 to 0.08. This supports the main setting: the replayed scaffold conditions the model state used for scoring, but copied evidence spans are selected from the original context rather than allocating the evidence budget over the entire prompt.

\subsection{Qualitative Results}

We further compare the methods through qualitative case analysis on retrieval-heavy examples, as shown in Figure~\ref{fig:case_study}. Vanilla receives the full context, but relevant spans may become weakly bound at generation time. AttnSharp and DySCO strengthen attention toward relevant tokens, but the selected evidence remains latent inside the decoding process rather than being exposed as text. A-MEM and DAC create shorter evidence views through memory retrieval or compression, which can help when the preprocessing step keeps the right span but can also weaken the answer if supporting sentences are omitted. By contrast, \name{} copies model-selected evidence sentences into an explicit evidence pool and replays them near the question while leaving the original context intact. In the inspected cases, the evidence pool appears to encourage answers grounded in a compact set of supporting sentences rather than in a diffuse long prompt.

\subsection{Time and Space Efficiency}

As shown in Figure~\ref{fig:runtime_comparison}, \name{} introduces the evidence-selection and replay stage, making it slightly slower than the vanilla baseline. However, it remains substantially faster than DySCO, which changes the backbone forward or decoding logic. In terms of GPU memory consumption, \name{} inserts fewer than 128 additional evidence tokens, resulting in only minimal memory overhead. Consequently, its overall memory consumption is similar to both Vanilla. Detailed measurements are provided in Appendix~\ref{app:efficiency_details}.

%% file: tables/thinking_ablation.tex
\begin{table}[t]
\centering
\small
\caption{Robustness evaluation with thinking enabled. \name{} remains strongest on NQ, PopQA and InfMC, indicating that evidence replay remains useful when the backbone performs explicit reasoning.}
\label{tab:thinking_ablation}
\vspace{-0.2cm}
\setlength{\tabcolsep}{3.5pt}
\renewcommand{\arraystretch}{1.22}
\begin{tabular}{lccccc}
\toprule
\multirow{2}{*}{\textbf{Method}} &
\multicolumn{2}{c}{\shortstack{\textbf{NQ}\\\textbf{128K}}} &
\multicolumn{2}{c}{\shortstack{\textbf{PopQA}\\\textbf{128K}}} &
\shortstack{\textbf{InfMC}\\\textbf{128K}} \\
\cmidrule(lr){2-3}
\cmidrule(lr){4-5}
\cmidrule(l){6-6}
& Acc & F1 & Acc & F1 & Acc \\
\midrule
Vanilla & 0.08 & 0.24 & 0.14 & 0.25 & 0.69 \\
AttnSharp & 0.13 & 0.27 & 0.14 & 0.24 & 0.63 \\
DAC & 0.10 & 0.23 & 0.13 & 0.24 & 0.66 \\
A-MEM & 0.09 & 0.23 & 0.11 & 0.23 & \underline{0.71} \\
DySCO & \underline{0.14} & \underline{0.29} & \underline{0.17} & \textbf{0.29} & 0.67 \\
\textbf{\name} & \textbf{0.15} & \textbf{0.30} & \textbf{0.18} & \underline{0.28} & \textbf{0.72} \\
\bottomrule
\end{tabular}
\end{table}

%% file: tables/robustness_64k_ablation.tex
\begin{table}[t]
\centering
\small
\caption{Robustness evaluation under a shorter context budget. \name{} stays within the top two on every metric, suggesting that the gains are not tied to the main context length.}
\label{tab:robustness_64k_ablation}
\vspace{-0.2cm}
\setlength{\tabcolsep}{3.5pt}
\renewcommand{\arraystretch}{1.22}
\begin{tabular}{lccccc}
\toprule
\multirow{2}{*}{\textbf{Method}} &
\multicolumn{2}{c}{\shortstack{\textbf{NQ}\\\textbf{64K}}} &
\multicolumn{2}{c}{\shortstack{\textbf{PopQA}\\\textbf{64K}}} &
\shortstack{\textbf{InfMC}\\\textbf{64K}} \\
\cmidrule(lr){2-3}
\cmidrule(lr){4-5}
\cmidrule(l){6-6}
& Acc & F1 & Acc & F1 & Acc \\
\midrule
Vanilla & 0.07 & 0.24 & 0.04 & 0.20 & 0.48 \\
AttnSharp & 0.07 & 0.24 & 0.03 & 0.20 & 0.44 \\
DAC & \underline{0.09} & 0.23 & \underline{0.17} & \textbf{0.32} & \underline{0.53} \\
A-MEM & 0.06 & 0.20 & 0.06 & 0.19 & 0.51 \\
DySCO & \textbf{0.11} & \underline{0.26} & 0.07 & 0.23 & 0.46 \\
\textbf{\name} & \textbf{0.11} & \textbf{0.26} & \textbf{0.18} & \underline{0.30} & \textbf{0.54} \\
\bottomrule
\end{tabular}
\end{table}

%% file: tables/token_source_ablation.tex
\begin{table}[t]
\centering
\small
\caption{Ablation on evidence-token source. Selecting evidence from the original context consistently outperforms selecting from the full replay prompt.}
\label{tab:token_source_ablation}
\vspace{-0.2cm}
\setlength{\tabcolsep}{3.5pt}
\renewcommand{\arraystretch}{1.22}
\begin{tabular}{lccccc}
\toprule
\multirow{2}{*}{\textbf{Source}} &
\multicolumn{2}{c}{\shortstack{\textbf{NQ}\\\textbf{128K}}} &
\multicolumn{2}{c}{\shortstack{\textbf{PopQA}\\\textbf{128K}}} &
\shortstack{\textbf{InfMC}\\\textbf{128K}} \\
\cmidrule(lr){2-3}
\cmidrule(lr){4-5}
\cmidrule(l){6-6}
& Acc & F1 & Acc & F1 & Acc \\
\midrule
Full prompt & \underline{0.04} & \underline{0.23} & \underline{0.02} & \underline{0.14} & \underline{0.52} \\
Context & \textbf{0.08} & \textbf{0.25} & \textbf{0.07} & \textbf{0.19} & \textbf{0.54} \\
\bottomrule
\end{tabular}
\end{table}

%% file: sections/conclusion.tex
\section{Conclusion}
\label{Conclusion}

We present \name{}, a training-free method that turns model-internal relevance signals into an explicit evidence scaffold for long-context reasoning. \name{} recursively selects candidate evidence and replays the organized scaffold before answer generation while preserving the original context, thereby separating evidence organization from answer generation. Our associative-memory analysis provides a simple interpretation of this process as cue-conditioned trace reactivation, and our experiments across eight 128K-context datasets and three backbones show consistent improvements over strong long-context baselines.
These findings suggest that long-context inference can be improved not only by extending context windows or compressing inputs, but also by better organizing the evidence already available in the prompt.

%% file: sections/appendix.tex
\appendix


\section{Extended Related Work}
\label{sec:appendix_extended_related_work}

\paragraph{Long-context modeling and evaluation.}
Recent long-context research has substantially expanded the input length that LLMs can process through efficient attention, position extrapolation, and long-context fine-tuning~\citep{beltagy2020longformer, zaheer2020bigbird, peng2023yarn, chen2024longlora, ding2023longnet, ding2024longrope, gcformer, ren2025riskporiskbasedpolicyoptimization,ning2026codeagentharness, guo2024moreembracingsparsityinterpolation, wei2026agenticreasoninglargelanguage}. In parallel, long-context benchmarks have shifted evaluation beyond input length alone, covering document understanding, multi-document reasoning, retrieval, summarization, code, and synthetic stress tests~\citep{bai2024longbench, shaham2023zeroscrolls, an2023leval, zhang2024infinitebench, hsieh2024ruler, yen2025helmet, zhao2026papermindbenchmarkingagenticreasoning}. These studies clarify an important distinction between context access and context utilization: a model may accept a long prompt while still failing to locate or use the evidence needed for a particular question. This gap is also reflected in analyses of position sensitivity, where relevant information can be underused when it appears in less favorable locations~\citep{liu2023lost}. \name{} addresses this utilization problem without changing the backbone model or extending its context window. It keeps the original context available, but adds a query-conditioned evidence scaffold that makes candidate supporting spans more accessible near generation time.

\paragraph{Retrieval, external memory, and prompt compression.}
Retrieval-augmented generation and retrieval-enhanced language modeling improve grounding by retrieving relevant passages or chunks before generation~\citep{lewis2021retrieval, borgeaud2022retro}, while recent agentic memory systems maintain and query external memory structures across interactions~\citep{xu2025amem}. Prompt and context compression methods take a different route: they reduce the amount of text passed to the model by filtering, pruning, or compressing less informative content~\citep{li2023selective, jiang2023llmlingua, jiang2024longllmlingua, zhao2025dac,zhao2025saberswitchablebalancedtraining}. These approaches can reduce distraction and computation, but they also introduce an additional selection or compression step that may omit useful evidence. \name{} is complementary to this line of work. It does not build a persistent memory, train a retriever, or replace the original long context with a shortened version. Instead, it uses the model's own prompt-internal signals to copy candidate evidence spans into an explicit scaffold, while leaving the full context in the prompt as a fallback source of information.

\paragraph{Attention-based inference and KV-cache methods.}
Several inference-time methods use internal attention patterns to improve long-context efficiency or utilization. Query-aware sparsity and dynamic attention-scaling methods use attention-derived signals to select relevant tokens or adjust attention behavior during decoding~\citep{tang2024questqueryawaresparsityefficient, ye2026dysco}. A related line of KV-cache methods observes that long-context attention often concentrates on heavy hitters, attention sinks, or salient key positions, and uses this structure for cache retention or eviction~\citep{liu2023scissorhand, zhang2023h2oheavyhitteroracleefficient, xiao2024efficientstreaminglanguagemodels, li2024snapkvllmknowslooking, cai2025pyramidkvdynamickvcache, feng2025adakvoptimizingkvcache, zhao2025secondorderfinetuningpainllmsa, dang2025fzoofastzerothorderoptimizer}. \name{} differs in both objective and mechanism. It does not directly rescale attention logits during final decoding and does not optimize a cache budget. More importantly, it treats attention only as an inexpensive proposal signal rather than as a faithful explanation of model behavior. The selected spans are materialized as readable text, so the final generation is conditioned on an explicit evidence scaffold rather than only on latent attention intervention.

\paragraph{Associative-memory interpretation.}
The behavior of \name{} can also be viewed through an associative-memory lens. Classical associative-memory models describe content-addressable retrieval from stored patterns, while dense and modern Hopfield-style formulations connect such retrieval with higher-capacity memory and attention-like updates~\citep{hopfield1982neural, krotov2016dense, ramsauer2021hopfieldnetworksneed, krotov2025modernmethodsassociativememory, zhao2024sparsevqtransformerffnfreeframework, chen2025doesvectorquantizationfail}. Related neural memory architectures, including Memory Networks and key-value memory networks, frame reasoning as query-conditioned access to stored representations or facts~\citep{weston2015memory, sukhbaatar2015endtoend, miller2016keyvalue}. Transformer-specific analyses further connect attention or feed-forward components with associative and key-value memory views~\citep{bricken2021attention, geva2021transformer}. In our setting, the long context acts as a collection of memory traces, the question serves as a retrieval cue, and attention provides a prompt-internal proxy for cue--trace association rather than a faithful explanation of the model's decision process. Evidence sifting selects candidate traces, evidence materialization turns them back into grounded text spans, and replay reactivates these spans near answer generation. Under this view, recursive evidence sifting is a lightweight way to repeatedly re-query the same context under a scaffold-conditioned state, improving query--evidence rebinding without adding an external memory module, training a retriever, or removing the original context.

\section{Experiment Details}
\label{sec:appendix_experiment_details}

\subsection{Dataset Details}
\label{sec:appendix_datasets}

The first seven datasets in our evaluation come from the HELMET benchmark~\cite{yen2025helmet}, and CLIPPER is evaluated as an additional long-context claim-verification benchmark~\cite{pham2025clipper}.\looseness-1
\begin{itemize}[itemsep=0.5ex, parsep=0pt, topsep=-2.3pt]
    \item \textbf{NQ} evaluates open-domain factual question answering, where the model must locate short answer evidence in a long input.
    \item \textbf{TriviaQA} also tests factual question answering, but with trivia-style questions that often require matching paraphrased clues to supporting evidence.
    \item \textbf{HotpotQA} focuses on multi-hop question answering and requires combining evidence from multiple pieces of context.
    \item \textbf{PopQA} probes entity-centric factual knowledge and is sensitive to whether relevant evidence about less prominent entities is correctly used.
    \item \textbf{NarrativeQA} evaluates narrative understanding over long stories, requiring the model to connect events, characters, and plot information.
    \item \textbf{InfBench QA} contains free-form question answering over very long inputs and stresses evidence retrieval from extended contexts.
    \item \textbf{InfBench MC} uses a multiple-choice format over long inputs, testing whether the model can select the option best supported by the context.
    \item \textbf{CLIPPER} evaluates claim verification over book-length contexts with evidence-grounded synthetic claims.
\end{itemize}

\subsection{GPU Resources}
\label{sec:appendix_gpu_resources}

We conduct experiments on NVIDIA A100 and NVIDIA H200 GPU servers. The H200 runs are executed on an ARM64 (\texttt{aarch64}) system architecture. All methods compared under the same setting use the same backbone, context budget, prompting format, and decoding configuration on the corresponding hardware.

\input{tables/efficiency_analysis}
\subsection{Detailed Efficiency Analysis}
\label{app:efficiency_details}
We report wall-clock runtime on CLIPPER using Llama3-8B at 128K context length with thinking disabled.
Vanilla full-context decoding takes 44 minutes, AttnSharp takes 46 minutes, DAC takes 34 minutes in total, A-MEM takes 50 minutes, and DySCO requires 2 hours and 13 minutes. The best-performing variant of \name{} takes 62 minutes.
The additional runtime of \name{} mainly comes from evidence sifting and replay, which add computation beyond direct full-context decoding; nevertheless, \name{} remains substantially faster than DySCO, which changes the backbone forward or decoding logic. Its GPU memory stays at the same level as Vanilla and DySCO because the replay scaffold adds fewer than 128 evidence tokens.

\begin{figure*}[t]
    \centering
    \includegraphics[width=\linewidth]{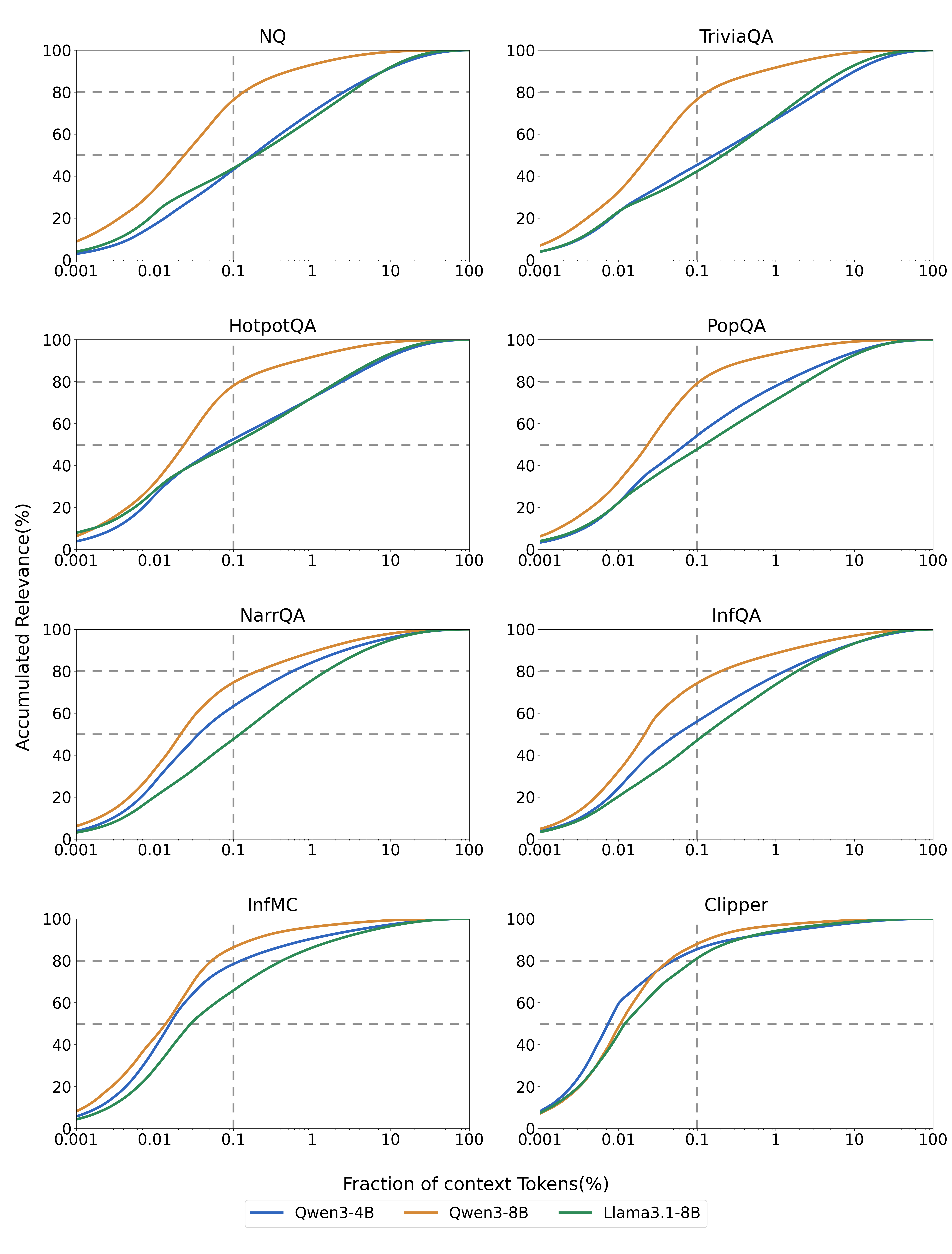}
    \caption{\textbf{Top 0.1\% of context tokens already accounts for about 50\% / 80\% accumulated relevance score across three LLMs, corresponding to only 128 tokens in a 128K-token context.} This figure ranks all context tokens by their relevance scores with respect to the question and shows how much accumulated relevance score is covered by the top-ranked tokens. Each curve represents the mean trend over eight datasets, and the shaded region shows the variance across datasets.}
    \label{fig:ablation_teaser}
\end{figure*}

\subsection{\name{} Implementation Details}
\label{sec:appendix_context_harness_details}

The \name{} implementation is activated by setting \texttt{--decoding\_method Our}. In this mode, the evaluation driver calls the sentence-replay path in \texttt{rescale\_generate}: the model receives the original prompt, computes candidate evidence proposals from attention, reconstructs a replay prompt, and then generates from the replay prompt. The same custom model classes also support attention-rescaling baselines, but \name{} uses the selected tokens to construct replayed evidence rather than to directly rescale attention logits during final decoding.\looseness-1

\paragraph{Prompt segmentation.}
For each dataset, the implementation identifies the boundary between the long-context portion and the question or answer-format suffix. Let $L_C$ be the number of tokens before this boundary. The original prompt is split into context tokens $x_{1:L_C}$ and question-side tokens $x_{L_C+1:T}$. The default replay prompt is formed as context, replayed evidence, and question-side tokens. The code also contains alternative replay positions for ablations, but the multi-round setting uses the before-question replay form.

\paragraph{Attention readout.}
The readout uses the final $w$ prompt tokens as cue tokens, where $w$ is controlled by \texttt{context\_warmup\_steps} and is set to 8 in our main runs. For each cue token, attention weights are averaged over a fixed set of selected layer-head pairs. Scores are accumulated across cue positions with exponential decay:
\begin{equation}
    r^{(t)} =
    \mathrm{Normalize}\left(a^{(t)} + \lambda r^{(t-1)}\right),
\end{equation}
where $a^{(t)}$ is the current averaged attention distribution and $\lambda$ is the decay factor. The configuration files set $\lambda=0.75$. Chat-template tokens are masked by default. In the main setting, candidate copied spans are restricted to the original context. The token-source ablation additionally allows candidate positions over the full replay prompt; in both settings, the replayed scaffold conditions the model state from which query-token attention is read out.

\paragraph{Token selection and sentence recovery.}
After scoring, the implementation applies top-$K$, top-$p$, or hybrid selection to obtain token positions. These selection rules follow prior token-importance methods and are not specific to \name{}. The method then decodes the prompt, finds sentence boundaries, and copies the sentence containing each selected token. Empty strings and trailing special markers are removed. When multiple selected tokens fall in the same sentence, the sentence is kept once.\looseness-1

\paragraph{Recursive evidence sifting.}
For $R$ replay rounds, the implementation repeats evidence proposal and sentence recovery on the current replay prompt. Newly recovered sentences are deduplicated against previously inserted sentences, accumulated into a single ordered scaffold, and inserted back between the original context and the question-side prompt. The final answer is generated only after the last replay round. The main scripts use $R=2$ with sentence wrapping disabled, so the inserted scaffold is a compact list of copied evidence sentences.\looseness-1

\paragraph{Caching and length accounting.}
To avoid unnecessary recomputation, the implementation snapshots the key-value cache at the end of the original context. During replay generation, it restores this context cache and processes only the inserted evidence plus the question-side tokens before decoding the answer. The generation length budget is extended by the number of inserted replay tokens so that adding evidence does not reduce the maximum number of answer tokens.

\section{Experimental Settings}
\label{sec:appendix_settings}

We use the same backbone model, context budget, prompting format, and answer decoding settings for \name{} and the corresponding baselines within each comparison. For Qwen3 models evaluated beyond their native context window, we enable YaRN rope scaling. The main benchmark comparison and most ablations use 128K contexts; the shorter-context robustness study in Table~\ref{tab:robustness_64k_ablation} uses 64K contexts. Thinking mode is disabled unless explicitly stated, with Table~\ref{tab:thinking_ablation} serving as the thinking-enabled robustness setting. All task scores in the experimental tables are reported as fractions in $[0,1]$, while average rank in Table~\ref{tab:main_results} remains on the original rank scale.\looseness-1

Dataset-specific maximum answer lengths and stopping behavior follow the evaluation scripts in the released code. The appendix ablations below isolate two hyperparameters on 128K NQ, PopQA, and InfBench MC: Table~\ref{tab:evidence_turn_ablation} varies the number of replay rounds $R$, and Table~\ref{tab:topk_ablation} varies the top-$K$ evidence-token candidate budget while keeping the other generation settings fixed.

\input{tables/evidence_turn_ablation}
\input{tables/topk_ablation}

\section{Additional Visualizations}
\label{sec:appendix_visualizations}

The extended robustness results are reported numerically in Tables~\ref{tab:thinking_ablation} and~\ref{tab:robustness_64k_ablation}. These tables complement the hyperparameter visualizations in Figure~\ref{fig:ablation_hyperparameters}: Table~\ref{tab:thinking_ablation} evaluates whether evidence replay remains useful when thinking is enabled at 128K context length, while Table~\ref{tab:robustness_64k_ablation} checks whether the same behavior holds under a shorter 64K context budget.

\section{Theoretical Analysis}

\subsection{Theoretical Setup}\label{app:thm:setup}

Following prior works  \cite{nichani2025understanding,olsson2022context}, we formulate the long-context task as follows. Suppose that each token $i=1,\dots,n$ in the context has mutually orthogonal embedding $c_i\in\mathbb R^d$ with $\|c_i\|_2=1$ and that the answer has embedding $y\in\mathbb R^d$. Initially, the prompt sequence is $x^{(0)}=[c_1,\dots,c_n]$. Let $q\in\mathbb R^d$ denote the query embedding. The initial attention scores are:
\begin{align}
a^{(0)}=\operatorname{softmax}\big([\langle x^{(0)}_i, q\rangle]_{i=1}^{n}\big).
\end{align}
The initial hidden embedding $h^{(0)}$ is:
\begin{align}
h^{(0)}=\sum_{i=1}^na^{(0)}_ix^{(0)}_i.
\end{align}
In each step $j\ge1$, we first append the most relevant evidence to the sequence:
\begin{align}
x^{(j)}=\big[x^{(j-1)},x^{(j-1)}_{\operatorname{argmax}(a^{(j-1)})}\big].
\end{align}
Then, we update the attention scores:
\begin{align}
a^{(j)}=\operatorname{softmax}\big([\langle x^{(j)}_i, h^{(j-1)}\rangle]_{i=1}^{n+j}\big).
\end{align}
Finally, we update the hidden embedding:
\begin{align}
h^{(j)}=\sum_{i=1}^{n+j}a^{(j)}_ix^{(j)}_i.
\end{align}
We assume that the context is relevant to the answer $y$: there exists $i^*$ such that $y=c_{i^*}$. We also assume that the query $q$ is relevant to the answer $y$: $\langle y,q\rangle>\langle c_i,q\rangle$ for all $i\ne i^*$ and $\max_{i\ne i^*}\frac{\langle y-c_i,q\rangle}{a^{(0)}_{i^*}-a^{(0)}_{i}}<1$.

\subsection{Proof of Theorem~\ref{thm:similarity}}\label{app:thm:proof}

\begin{proof}
To prove that the cosine similarity $\cos(h^{(j)}, y)$ strictly increases with each step $j \ge 1$, we will trace the evolution of the attention weights for each token.


W.l.o.g., suppose that $i^*=1$, so $y = c_1$. The initial sequence is $x^{(0)} = [c_1, \dots, c_n]$.
Let $w^{(j)}_i$ denote the sum of the attention weights of all copies of the token $c_i$ in the sequence $x^{(j)}$ when computing $h^{(j)}$. Thus, we can express the hidden embedding at any step $j$ as:
\begin{align}
h^{(j)} = \sum_{i=1}^n w^{(j)}_i c_i.
\end{align}
Since the embeddings $c_i$ are mutually orthogonal and have unit norm ($\|c_i\|_2 = 1$), the cosine similarity between $h^{(j)}$ and $y = c_1$ is:
\begin{align}
\cos(h^{(j)}, y) &= \frac{\langle h^{(j)}, c_1 \rangle}{\|h^{(j)}\|_2 \|c_1\|_2} \\&= \frac{w^{(j)}_1}{\sqrt{\sum_{m=1}^n (w^{(j)}_m)^2}} \\&= \left( 1 + \sum_{i \ne 1} \left( \frac{w^{(j)}_i}{w^{(j)}_1} \right)^2 \right)^{-1/2}.
\end{align}
Let $R^{(j)}_i = \frac{w^{(j)}_1}{w^{(j)}_i}$. To prove that $\cos(h^{(j)}, y) > \cos(h^{(j-1)}, y)$, it is sufficient to prove that the ratio $R^{(j)}_i > R^{(j-1)}_i$ for all $i \ne 1$ and for all $j \ge 1$. We will use induction to prove this.


At $j=0$, the attention weights are $w^{(0)}_i = a^{(0)}_i$.
The initial ratio for the base state (as dictated by the dot product) can be defined as $R^{(0)}_i = e^{\langle c_1 - c_i, q \rangle}$.
Since $\langle y - c_i, q \rangle < a^{(0)}_1 - a^{(0)}_i$, then
\begin{align}
\ln R^{(0)}_i < w^{(0)}_1 - w^{(0)}_i.
\end{align}
We are also given that $\langle y, q \rangle > \langle c_i, q \rangle$ for all $i \ne 1$, meaning $a^{(0)}_1 > a^{(0)}_i$. Thus, the most relevant evidence token at step 0 is $c_1$, and $x^{(1)}$ appends $c_1$.
At step 1, the sequence has $2$ copies of $c_1$ and $1$ copy of $c_i$. The unnormalized attention score for each copy of $c_k$ is $e^{\langle c_k, h^{(0)} \rangle} = e^{w^{(0)}_k}$. Normalizing these gives:
\begin{align}
w^{(1)}_1 = \frac{2 e^{w^{(0)}_1}}{Z_1}, \quad w^{(1)}_i = \frac{e^{w^{(0)}_i}}{Z_1},
\end{align}
where $Z_j=\sum_{i=1}^{n}w^{(j)}_i$ is the denominator of softmax in attention scores.
Evaluating the ratio $R^{(1)}_i$:
\begin{align}
R^{(1)}_i = \frac{w^{(1)}_1}{w^{(1)}_i} = 2 e^{w^{(0)}_1 - w^{(0)}_i}.
\end{align}
We want to show $R^{(1)}_i > R^{(0)}_i$. Substituting our assumption $\ln R^{(0)}_i < w^{(0)}_1 - w^{(0)}_i$:
\begin{align}
R^{(1)}_i = 2 e^{w^{(0)}_1 - w^{(0)}_i} > 2 e^{\ln R^{(0)}_i} = 2 R^{(0)}_i > R^{(0)}_i.
\end{align}
Thus, $R^{(1)}_i > R^{(0)}_i$, which implies $\cos(h^{(1)}, y) > \cos(h^{(0)}, y)$.


Assume for step $j-1$ that $R^{(j-1)}_i > R^{(j-2)}_i > \dots > R^{(0)}_i > 1$.
Because $R^{(j-1)}_i > 1$, we have $w^{(j-1)}_1 > w^{(j-1)}_i$. Therefore, individual tokens $c_1$ continue to command the highest attention scores, meaning $c_1$ is consistently appended. At step $j$, there are $N^{(j)}_1 = j+1$ copies of $c_1$ and $1$ copy of each $c_i$.

Let $\Delta^{(j)}_i = w^{(j)}_1 - w^{(j)}_i$. The weights update according to:
\begin{align}
w^{(j)}_1 = \frac{(j+1) e^{w^{(j-1)}_1}}{Z_j}, \quad w^{(j)}_i = \frac{e^{w^{(j-1)}_i}}{Z_j}.
\end{align}
Dividing the two yields the recurrence relation for $R$:
\begin{align}
R^{(j)}_i = (j+1) e^{\Delta^{(j-1)}_i}.
\end{align}
We want to prove $R^{(j)}_i > R^{(j-1)}_i$. Substituting the recurrence for $R^{(j-1)}_i = j e^{\Delta^{(j-2)}_i}$, the condition becomes:
\begin{align}
(j+1) e^{\Delta^{(j-1)}_i} > j e^{\Delta^{(j-2)}_i},
\end{align}
which is equivalent to:
\begin{align}
\Delta^{(j-1)}_i > \Delta^{(j-2)}_i - \ln\left(1 + \frac{1}{j}\right).
\end{align}
We will in fact prove a strictly stronger statement: $\Delta^{(j)}_i > \Delta^{(j-1)}_i$ for all $j \ge 1$.

By writing $w^{(j)}_1$ and $w^{(j)}_i$ explicitly using the sum $Z_j$, we can express $\Delta^{(j)}_i$ as:
\begin{align}
\Delta^{(j)}_i = \frac{j+1 - e^{-\Delta^{(j-1)}_i}}{j+1 + \Sigma^{(j-1)}},
\end{align}
where $\Sigma^{(j-1)} = \sum_{m \ne 1} e^{-\Delta^{(j-1)}_m}$.
Let $u = \Delta^{(j-1)}_i$ and $v = \Delta^{(j-2)}_i$. We evaluate the difference $u - \Delta^{(j)}_i$:
\begin{align}
&u - \Delta^{(j)}_i = u - \frac{j+1 - e^{-u}}{j+1 + \Sigma^{(j-1)}} \\={}& \frac{u(j+1 + \Sigma^{(j-1)}) - (j+1) + e^{-u}}{j+1 + \Sigma^{(j-1)}}.
\end{align}
Let's analyze the numerator $u(j+1 + \Sigma^{(j-1)}) - (j+1) + e^{-u}$. From our previous step, $u$ is formulated as $u = \frac{j - e^{-v}}{j + \Sigma^{(j-2)}}$.
By the inductive hypothesis $\Delta^{(j-1)}_m > \Delta^{(j-2)}_m - \ln(1 + \frac{1}{j})$, we have $e^{-\Delta^{(j-1)}_m} < \frac{j+1}{j} e^{-\Delta^{(j-2)}_m}$.
Summing this over all $m \ne 1$ gives bounds on $\Sigma$:
\begin{align}
\Sigma^{(j-1)} < \frac{j+1}{j} \Sigma^{(j-2)}.
\end{align}
Using this upper bound, we can bound the term $u(j+1+\Sigma^{(j-1)})$ in the numerator:
\begin{align}
&u(j+1 + \Sigma^{(j-1)}) \\
<{}& \left( \frac{j - e^{-v}}{j + \Sigma^{(j-2)}} \right) \left( j+1 + \frac{j+1}{j} \Sigma^{(j-2)} \right) \\
={}& \left( \frac{j - e^{-v}}{j + \Sigma^{(j-2)}} \right) \frac{j+1}{j} \left( j + \Sigma^{(j-2)} \right) \\
={}& \frac{j+1}{j} (j - e^{-v}) = j+1 - \frac{j+1}{j} e^{-v}.
\end{align}
Substituting this back into the numerator of our difference expression, we have that the numerator is smaller than:
\begin{align}
<{}&\left( j+1 - \frac{j+1}{j} e^{-v} \right) - (j+1) + e^{-u} \\
={}& e^{-u} - \frac{j+1}{j} e^{-v}.
\end{align}
Because we assumed $\Delta^{(j-1)}_i > \Delta^{(j-2)}_i - \ln(1+\frac{1}{j})$, it is an algebraic consequence that $e^{-u} < \frac{j+1}{j} e^{-v}$. Therefore, the numerator is $<0$, which implies $u - \Delta^{(j)}_i<0$. It follows that
\begin{align}
\Delta^{(j)}_i > \Delta^{(j-1)}_i.
\end{align}


Because $\Delta^{(j)}_i$ is strictly increasing, it easily satisfies the required bound $\Delta^{(j)}_i > \Delta^{(j-1)}_i - \ln(1 + \frac{1}{j+1})$. This guarantees that:
\begin{align}
R^{(j+1)}_i > R^{(j)}_i.
\end{align}
Since the ratio of the correct answer's weight to every incorrect answer's weight strictly increases at every step $j$, the relative mass of $w^{(j)}_1$ continuously approaches $1$. Consequently, the denominator in the cosine similarity formula strictly shrinks, yielding:\looseness-1
\begin{align}
&\cos(h^{(j)}, y) > \cos(h^{(j-1)}, y), \quad \forall j \ge 1.\qedhere
\end{align}
\end{proof}

%% file: tables/efficiency_analysis.tex
\begin{table}[t]
\centering
\small
\caption{Runtime usage on CLIPPER using Llama3-8B at 128K context length.}
\label{tab:efficiency_analysis}
\setlength{\tabcolsep}{5pt}
\resizebox{0.8\linewidth}{!}{
\begin{tabular}{lc}
\toprule
\textbf{Method} & \textbf{Runtime} \\ 
\midrule
Vanilla & 44 min \\ 
AttnSharp & 46 min \\ 
\shortstack[l]{DAC} & 34 min \\ 
A-MEM & 50 min \\ 
DySCO & 2h 13min \\ 
\textbf{\name{}} & 62 min \\ 
\bottomrule
\end{tabular}
}
\end{table}

%% file: tables/evidence_turn_ablation.tex
\begin{table}[t]
\centering
\small
\caption{Ablation on recursive evidence-sifting rounds. Multiple rounds substantially improve over a single round, while the best depth varies by dataset.}
\label{tab:evidence_turn_ablation}
\renewcommand{\arraystretch}{1.22}
\setlength{\tabcolsep}{3.5pt}
\begin{tabular}{lccccc}
\toprule
\multirow{2}{*}{\textbf{$R$ Rounds}} &
\multicolumn{2}{c}{\shortstack{\textbf{NQ}\\\textbf{128K}}} &
\multicolumn{2}{c}{\shortstack{\textbf{PopQA}\\\textbf{128K}}} &
\shortstack{\textbf{InfMC}\\\textbf{128K}} \\
\cmidrule(lr){2-3}
\cmidrule(lr){4-5}
\cmidrule(l){6-6}
& Acc & F1 & Acc & F1 & Acc \\
\midrule
1 & 0.04 & 0.21 & 0.01 & 0.10 & 0.48 \\
2 & \underline{0.08} & \underline{0.25} & \textbf{0.07} & \textbf{0.19} & 0.50 \\
3 & \textbf{0.09} & \underline{0.25} & \underline{0.05} & \underline{0.18} & \underline{0.51} \\
4 & \textbf{0.09} & \textbf{0.25} & \underline{0.05} & 0.17 & \textbf{0.54} \\
\bottomrule
\end{tabular}
\end{table}

%% file: tables/topk_ablation.tex
\begin{table}[t]
\centering
\small
\caption{Ablation on the evidence candidate budget. Larger candidate sets improve PopQA and InfMC, but can hurt NQ, revealing a recall--noise trade-off.}
\label{tab:topk_ablation}
\setlength{\tabcolsep}{3.5pt}
\renewcommand{\arraystretch}{1.22}
\begin{tabular}{lccccc}
\toprule
\multirow{2}{*}{\textbf{Top-$K$}} &
\multicolumn{2}{c}{\shortstack{\textbf{NQ}\\\textbf{128K}}} &
\multicolumn{2}{c}{\shortstack{\textbf{PopQA}\\\textbf{128K}}} &
\shortstack{\textbf{InfMC}\\\textbf{128K}} \\
\cmidrule(lr){2-3}
\cmidrule(lr){4-5}
\cmidrule(l){6-6}
& Acc & F1 & Acc & F1 & Acc \\
\midrule
1 & 0.03 & 0.22 & 0.04 & 0.14 & 0.52 \\
8 & \textbf{0.08} & \textbf{0.25} & \underline{0.07} & \underline{0.19} & 0.50 \\
16 & \textbf{0.08} & \underline{0.25} & 0.05 & 0.17 & \underline{0.55} \\
32 & \underline{0.04} & 0.23 & \textbf{0.10} & \textbf{0.21} & \textbf{0.58} \\
\bottomrule
\end{tabular}
\end{table}